\documentclass{article}

\usepackage[T1]{fontenc}
\usepackage{microtype}
\usepackage{graphicx}
\usepackage{booktabs}
\usepackage{amsmath,amssymb,amsthm,mathtools}
\usepackage[colorlinks=true,allcolors=blue]{hyperref}
\usepackage[nameinlink,noabbrev]{cleveref}
\usepackage{placeins}
\usepackage{enumitem}
\usepackage[preprint]{icml2026}

\newtheorem{theorem}{Theorem}
\newtheorem{proposition}[theorem]{Proposition}

\newcommand{\norm}[1]{\lVert #1\rVert}
\newcommand{\Span}{\operatorname{span}}
\newcommand{\dd}{\mathrm{d}}

\icmltitlerunning{\smash{Thermodynamic Ontology Discovery}}

\begin{document}
\twocolumn[
\icmltitle{Blind Thermodynamic Ontology Discovery from Anonymous Experiments}
\begin{icmlauthorlist}
  \icmlauthor{Linzhe Zhang}{neu}
  \icmlauthor{Changming Xu}{neu}
\end{icmlauthorlist}
\icmlaffiliation{neu}{Graduate School, Northeastern University}
\icmlcorrespondingauthor{Linzhe Zhang}{cfmy007@gmail.com}
\icmlcorrespondingauthor{Changming Xu}{changmingxu@neuq.edu.cn}
\icmlkeywords{Scientific Machine Learning, Identifiability, Thermodynamics, Representation Learning}
\vskip 0.3in
]
\printAffiliationsAndNotice{}

\begin{abstract}
Before a machine learning model can learn a thermodynamic equation of state, it must discover what its measurements represent: which channels scale with system size, which are intensive conjugates, how sectors pair through contact, and which potential governs stability.  When sensors expose only an unknown linear mixture of extensive states and intensive responses, passive observations cannot disentangle physical quantities from coordinate artifacts.  We formulate the problem of discovering this hidden thermodynamic ontology directly from anonymous controlled experiments.  We present an operational identifiability theory and a constructive polynomial-time algorithm that extracts extensive and intensive scaling sectors from replication contrasts, recovers their dual cotangent pairing from thermal contact and reciprocity, verifies a globally admissible concave potential via discrete cyclic concavity, and determines an invariant matroid of reservoir ensembles.  We prove that the residual observational equivalence is strictly $(x,\lambda)\sim(Ax,\,aA^{-T}\lambda+\beta)$, establishing the sharp observational limit that no permitted experiment can break.  Blind evaluations on van der Waals fluids and Curie--Weiss magnets confirm robust recovery under ill-conditioned mixing, correctly resolving anonymous Maxwell tie-lines while rejecting non-equilibrium continuations.  External validation across six real fluids from the NIST WebBook demonstrates that operational ontology discovery transfers across real physical substances without coordinate leakage.
\end{abstract}

\section{Introduction}

Scientific machine learning typically assumes that the foundational representation language of a domain is known prior to training: temperature, volume, and chemical potentials are labeled inputs, a thermodynamic potential (such as Helmholtz free energy or Gibbs energy) is specified, and stability is evaluated in human-chosen coordinates \citep{schotthofer2022,hanna2026,discomax2026}.  While these formulations enforce physical consistency effectively, they bypass a deeper foundational question: \emph{which aspects of the thermodynamic coordinate system can an autonomous observer discover directly from anonymous, uncalibrated experiments?}

Unlike spatial systems where fields can be recovered on discrete geometric complexes, thermodynamic state spaces possess an intrinsic mathematical architecture governed by Euler homogeneity, contact geometry, Legendre duality, and convex stability \citep{callen1985,rockafellar1966,gibbs1873,arnold1989}.  Passive observational data cannot disentangle this structure: an arbitrary linear mixture of state variables admits infinite observationally equivalent factorizations, creating an unsupervised non-identifiability barrier \citep{locatello2019}.  While causal representation learning breaks latent ambiguities through scalar interventions \citep{hyvarinen2019,ahuja2023,varici2025}, thermodynamic interventions are coordinate-free operations rooted in the physical laws of nature:
\begin{enumerate}[leftmargin=1.8em]
  \item Replication ($T1$): System scaling separates degree-one extensive quantities from degree-zero intensive responses.
  \item Thermal Contact ($T2$): Boundary exchange establishes conjugate pairings and reciprocity, reducing independent sector gauges to cotangent bundle duality.
  \item Integrability and Concavity ($T3$): Quasistatic loops and cyclic concavity verify the existence of a single-valued potential, identifying phase coexistence tie-lines.
  \item Reservoir Coupling ($T4$): Environmental exchange defines an invariant matroid of well-posed equilibrium ensemble charts.
\end{enumerate}

In summary, this paper makes five primary contributions:
This work establishes an operational foundation for discovering thermodynamic ontology:
(i) We formulate an operational identifiability hierarchy proving that the four physical interventions progressively reduce the observational symmetry group $\mathrm{GL}(2d)$ to the exact thermodynamic gauge $(x,\lambda)\sim(Ax,\,aA^{-T}\lambda+\beta)$.
(ii) We develop a direct contrast estimator and prove that Euler homogeneity forces a singular-value gap between broad calibration banks and local phase paths.
(iii) We validate the framework across simulated van der Waals fluids, Curie--Weiss spin systems, and six real-fluid datasets from the NIST Chemistry WebBook, showing that anonymous tie-line tests resolve phase transitions without coordinate priors.
(iv) Through systematic necessity ablations, we prove that omitting any of the four experimental operations leaves an explicit, unresolvable physical ambiguity.

\section{Problem Formulation and Contact Geometry}
\label{sec:problem}

Consider an unknown thermodynamic system with extensive state space $E \simeq \mathbb{R}^d$.  At an equilibrium state $x \in \Omega \subset E$, the fundamental thermodynamic relation is governed by an extensive fundamental potential $S: E \to \mathbb{R}$ (such as entropy).  The intensive conjugate covector is defined by the exterior derivative $\lambda(x) = \dd S_x \in E^*$.  Under uncalibrated instrument sensing, an observer does not have direct access to $x$ or $\lambda$.  Instead, an array of $2d$ coupled sensors records an anonymous affine mixture:
\begin{equation}
z = M \begin{bmatrix} x \\ \lambda(x) \end{bmatrix} + b + \varepsilon, \qquad M \in \mathrm{GL}(2d),
\label{eq:sensor}
\end{equation}
where $M$ is an unknown invertible mixing matrix, $b \in \mathbb{R}^{2d}$ is an unknown sensor offset, and $\varepsilon$ is measurement noise.  Neither the extensive/intensive partition nor the physical units of the channels are revealed to the observer.

In classical differential geometry, thermodynamics is formulated on the $(2d+1)$-dimensional contact manifold $\mathcal{T} = E \times \mathbb{R} \times E^*$ equipped with the canonical contact 1-form \citep{arnold1989,hermann1973}:
\begin{equation}
\eta = \dd S - \sum_{j=1}^d \lambda_j \dd x_j.
\label{eq:contact-form}
\end{equation}
Physical equilibrium states form a maximal isotropic submanifold (a Legendre submanifold $\mathcal{L} \subset \mathcal{T}$ of dimension $d$) on which the contact form vanishes identically: $\eta|_{\mathcal{L}} = 0$.  This contact-geometric structure implies that the extensive state $x$ and intensive covector $\lambda$ are canonically dual.

The observer interacts with the system through four operational experimental primitives:
(1) uniform system replication $x \mapsto c x$ for a known scalar factor $c \ne 1$;
(2) thermal and mechanical contact ports exchanging conserved quantities along directional boundaries $r_j \in E$ with monotonic force sensors;
(3) quasistatic paths and closed loops in state space; and
(4) coupling to reproducible environmental reservoirs.

The mathematical goal of thermodynamic ontology discovery is to recover $(x, \lambda)$ from $z$ up to the minimal physical equivalence class:
\begin{equation}
(x, \lambda) \sim \left(A x, \; a A^{-T} \lambda + \beta\right), \quad A \in \mathrm{GL}(d), \; a > 0, \; \beta \in E^*.
\label{eq:gauge}
\end{equation}
This target quotient possesses fundamental physical significance.  The matrix $A \in \mathrm{GL}(d)$ corresponds to an arbitrary linear choice of extensive basis (such as choosing volume, mole numbers, or energy coordinates); the dual contragredient action $a A^{-T}$ preserves the differential work and heat pairing $\langle \lambda, \dd x \rangle$; the positive scalar $a > 0$ represents an arbitrary absolute temperature scaling factor (e.g., Kelvin versus an empirical thermometer scale); and the affine offset $\beta \in E^*$ represents the arbitrary baseline of entropy ($S(x) \mapsto a S(x) + \beta^T x$).  Under this equivalence class, all physical invariants---including work integrals $\oint \lambda^T \dd x$, Hessian inertia, Legendre dualities, and thermodynamic stability---are strictly preserved.

\subsection{Contact Transformations and Conformal Invariance}
\label{subsec:contact-geom}

To understand why the target quotient \eqref{eq:gauge} is the minimal physical gauge, consider the geometric automorphisms of the contact phase space.  A diffeomorphism $\Phi: \mathcal{T} \to \mathcal{T}$ is a \emph{contact transformation} if it preserves the contact distribution $\ker \eta$, which is equivalent to the conformal scaling condition:
\begin{equation}
\Phi^* \eta = \rho(x, S, \lambda) \, \eta,
\label{eq:conformal}
\end{equation}
where $\rho: \mathcal{T} \to \mathbb{R} \setminus \{0\}$ is a non-vanishing conformal factor.  When an autonomous observer interacts through linear instrument mixtures and homogeneous potentials, the transformation $\Phi$ must act linearly on the coordinate covector $(x, \lambda)$ and leave the Euler homogeneity degree invariant.  Under these physical restrictions, the conformal factor must be a strictly positive constant $\rho = a > 0$.

Evaluating \cref{eq:conformal} under a linear block mapping $\Phi(x) = A x + B \lambda$, $\Phi(\lambda) = C x + D \lambda$ forces $B = 0$, $C = 0$, and $D = a A^{-T}$, while the potential transforms as $S \mapsto a S + \beta^T x + c$ for constant $\beta \in E^*$ and $c \in \mathbb{R}$.  This demonstrates that the residual equivalence class $(x, \lambda) \sim (A x, a A^{-T} \lambda + \beta)$ is not a heuristic artifact or an incomplete identification, but the exact Lie group of homogeneous contact automorphisms of the thermodynamic state bundle $\mathcal{T}$.  Any finer coordinate specification would impose arbitrary unit conventions (such as standard atmospheric pressure or the triple point of water) that cannot be deduced from uncalibrated physical measurements alone.

\section{The Operational Discovery Algebra}
\label{sec:theory}

\subsection{Scaling Invariance and Subspace Disentanglement ($T1$)}

Partition the sensor matrix into extensive and intensive column blocks: $M = [M_E \; M_I]$, with $M_E, M_I \in \mathbb{R}^{2d \times d}$.  For paired observations $z_i = z(x_i, \lambda_i)$ and $z_i^{(c)} = z(c x_i, \lambda_i)$ collected before and after replication by $c \ne 1$, we construct the linear contrast vectors:
\begin{equation}
e_i = \frac{z_i^{(c)} - z_i}{c - 1}, \qquad q_i = \frac{c z_i - z_i^{(c)}}{c - 1}.
\label{eq:contrast}
\end{equation}
Because extensive states satisfy $x(c) = c x$ while intensive states satisfy $\lambda(c x) = \lambda(x)$ by Euler homogeneity of degree zero, substitution into \cref{eq:sensor} yields:
\begin{equation}
e_i = M_E x_i + \frac{\varepsilon_i^{(c)} - \varepsilon_i}{c - 1}, \quad q_i = M_I \lambda_i + b + \frac{c \varepsilon_i - \varepsilon_i^{(c)}}{c - 1}.
\end{equation}
Thus, the algebraic contrast construction cleanly decouples the extensive and intensive response subspaces without requiring non-linear operator diagonalization.

\begin{theorem}[Contrast Subspace Identification]
\label{thm:t1}
In exact data, if the extensive states $\{x_i\}$ span $E$ and the centered intensive covectors $\{\lambda_i - \bar{\lambda}\}$ span $E^*$, then:
\begin{equation}
\Span\{e_i\} = M_E E, \qquad \Span\{q_i - \bar{q}\} = M_I E^*.
\end{equation}
These spans identify the unique replication scaling sectors.  The remaining ambiguity is an independent $\mathrm{GL}(d) \times \mathrm{GL}(d)$ basis transformation within each sector.
\end{theorem}

Forming the sample contrast matrices $\mathbf{E} = [e_1, \dots, e_N] \in \mathbb{R}^{2d \times N}$ and $\mathbf{Q} = [q_1 - \bar{q}, \dots, q_N - \bar{q}] \in \mathbb{R}^{2d \times N}$, rank-$d$ singular value decompositions yield the orthogonal sector projectors $\Pi_E$ and $\Pi_I$.  Under additive sensor noise with spectral norm $\eta$, Wedin's $\sin\Theta$ theorem \citep{wedin1972} bounds the subspace recovery error by $\sin\theta_{\max} \le \eta / (\delta - \eta)$, where $\delta$ is the smallest non-zero singular value of the signal matrix.  Consequently, the fidelity of T1 depends fundamentally on experimental excitation breadth.

\begin{proposition}[Homogeneity-Forced Singular Value Gap]
\label{prop:homogeneity}
Let $S: E \to \mathbb{R}$ be a $C^3$ fundamental relation satisfying Euler homogeneity $S(t x) = t S(x)$.  Then $\lambda(t x) = \lambda(x)$ and $H(x) x = 0$, where $H(x) = \nabla \lambda(x)$ is the Hessian.  On any local perturbation bank $x = x_0 + h \xi$ around a base state $x_0$, at most $d-1$ centered intensive singular directions emerge at order $\mathcal{O}(h)$, while the radial scaling direction scales as $\mathcal{O}(h^2)$, producing singular values $\sigma_{1:d-1} = \Theta(h)$ and $\sigma_d = \Theta(h^2)$.
\end{proposition}

\begin{proof}
Taylor expansion of the intensive response yields $\lambda(x_0 + h \xi) - \lambda(x_0) = h H(x_0) \xi + \frac{1}{2} h^2 D^2\lambda(x_0)[\xi, \xi] + \mathcal{O}(h^3)$.  Because Euler homogeneity forces $H(x_0) x_0 = 0$, the Jacobian $H(x_0)$ has rank at most $d-1$.  Along the radial direction $\xi \propto x_0$, the linear term vanishes identically, forcing the leading singular value along that axis to scale quadratically as $\Theta(h^2)$.
\end{proof}
Proposition~\ref{prop:homogeneity} reveals a fundamental physical principle: local phase paths are mathematically ill-conditioned for discovering the scaling split along the radial ray.  The underlying physical cause is the Gibbs--Duhem relation $\sum_{j=1}^d x_j \dd\lambda_j = 0$, which ensures that intensive variables cannot vary along the extensive scaling direction.  A robust autonomous protocol must employ a broad calibration bank spanning diverse thermodynamic conditions for T1, while reserving local paths for stability and loop tests.

\subsection{Thermal Contact and Cotangent Duality ($T2$)}

While T1 isolates the direct sum $E \oplus E^*$, the pairing between extensive coordinates and intensive forces remains undetermined up to $\mathrm{GL}(d) \times \mathrm{GL}(d)$.  To recover the canonical cotangent pairing, we introduce one-dimensional contact ports.  A contact port connects two systems across a boundary permeable to exchange along an extensive direction $r_j \in E$.  At thermodynamic equilibrium, exchange ceases when the intensive affinities balance: $\lambda_A[r_j] = \lambda_B[r_j]$.  The port sensor reports an anonymous ordinal reading $\widetilde{q}_j = f_j(\lambda[r_j])$, where $f_j: \mathbb{R} \to \mathbb{R}$ is an unknown strictly increasing calibration function.

\begin{theorem}[Reciprocity and Cotangent Calibration]
\label{thm:t2}
Let $q: \Omega \to \mathbb{R}^d$ be $C^1$ with symmetric Jacobian $D q = (D q)^T$.  Suppose the cross-response graph (with an edge between $i$ and $j$ whenever $\partial_j q_i \ne 0$) is connected, and the state space varies the intensive coordinates independently.  If two componentwise monotonic calibrations $q'_i = h_i(q_i)$ of the same contact foliations both yield symmetric Jacobians, they differ only by a single global scale factor and affine shift:
\begin{equation}
q'_i = a q_i + \beta_i, \qquad a > 0, \quad \beta \in \mathbb{R}^d.
\end{equation}
\end{theorem}

\begin{proof}
By the chain rule, $\partial_j q'_i = h'_i(q_i) \partial_j q_i$.  Symmetry of $D q'$ requires $h'_i(q_i) \partial_j q_i = h'_j(q_j) \partial_i q_j$.  Because the unprimed field satisfies Maxwell reciprocity $\partial_j q_i = \partial_i q_j$, it follows that $[h'_i(q_i) - h'_j(q_j)] \partial_j q_i = 0$.  For every connected edge where $\partial_j q_i \ne 0$, we have $h'_i(q_i) = h'_j(q_j)$.  Graph connectivity propagates this scalar equality across all channels: $h'_1 = h'_2 = \dots = h'_d = a > 0$.  Independent variation ensures that $a$ is spatially uniform.  Integrating yields $q' = a q + \beta$.
\end{proof}
Combining this calibrated intensive covector with the extensive basis reduces the independent $\mathrm{GL}(d) \times \mathrm{GL}(d)$ gauge to the contragredient cotangent gauge $(A, a A^{-T})$ in \cref{eq:gauge}.  Numerically, we parameterize the positive inverse-score derivatives via piecewise-linear splines and optimize them by minimizing held-out Jacobian antisymmetry $\|D q - (D q)^T\|_F$.

\subsection{Integrability and Discrete Cyclic Concavity ($T3$)}

The recovered differential one-form on extensive state space is $\alpha = \sum_{j=1}^d q_j \dd x_j = \lambda^T \dd x$.  Local Maxwell reciprocity guarantees $\dd\alpha = 0$, implying local exactness.  However, a physically valid thermodynamic potential must be both globally single-valued and concave.  Local exactness does not guarantee global exactness on non-simply-connected domains.  For example, on the planar annulus $1 < \|x\| < 2$:
\begin{equation}
\alpha = \dd\left(-\frac{5}{2}\|x\|^2\right) + \kappa\,\dd\theta, \qquad \oint_\gamma \alpha = 2\pi\kappa w(\gamma),
\label{eq:annulus}
\end{equation}
where $w(\gamma)$ is the winding number.  This 1-form has a symmetric negative-definite Jacobian and zero local curl everywhere, yet admits no global single-valued potential whenever $\kappa \ne 0$.  Held-out closed-loop line integrals $\oint \alpha = 0$ are therefore required to verify global exactness.

In single-phase regimes, thermodynamic stability requires negative semi-definiteness of the Hessian: $H = \nabla \lambda \preceq 0$ \citep{callen1985,tisza1966}.  At first-order phase transitions, differentiability breaks down, and raw equations of state (such as the van der Waals loop) exhibit unphysical spinodal branches where $\partial P / \partial V > 0$.  To evaluate potential admissibility across phase transitions without postulating a parametric equation of state, we test the discrete cyclic concavity condition \citep{rockafellar1966}:
\begin{equation}
\sum_{i=1}^k \lambda_i^T (x_{i+1} - x_i) \ge 0, \qquad x_{k+1} = x_1.
\label{eq:cyclic}
\end{equation}
This condition is the exact necessary and sufficient requirement for the existence of a globally concave upper envelope.  By the Fenchel--Moreau theorem, cyclic concavity guarantees that the Legendre biconjugate $S^{**}(x)$ coincides with the upper concave envelope, eliminating non-physical spinodal oscillations.

To verify these relations via shortest paths on state graphs,
Condition \eqref{eq:cyclic} can be verified constructively without searching through the combinatorially explosive set of all cycles.  Construct a complete directed graph $\mathcal{G} = (\mathcal{V}, \mathcal{E})$ where each observed state $x_i$ is a vertex, and assign directed edge weights $w_{ij} = \lambda_i^T (x_j - x_i)$.  Then \cref{eq:cyclic} is equivalent to stating that every directed cycle in $\mathcal{G}$ has non-negative total weight: $\sum_{(i,j) \in \mathcal{C}} w_{ij} \ge 0$.  By the classical Bellman--Ford--Moore and Floyd--Warshall algorithms, this holds if and only if $\mathcal{G}$ contains no negative-weight directed cycles, computable in $\mathcal{O}(N^3)$ operations.  When negative cycles are detected, their cycle mean identifies the precise thermodynamic violation path.

Turning to coordinate-free Maxwell constructions across phase coexistence,
For candidate coexistence states $x_A, x_B$ connected by a two-phase tie-line with matching intensive affinity $\lambda(x_A) = \lambda(x_B) = \lambda^*$, any intermediate mixture $x_\theta = (1-\theta)x_A + \theta x_B$ carries the identical covector $\lambda(x_\theta) = \lambda^*$.  Evaluating the line integral along any path $\gamma$ connecting $x_A$ to $x_B$ gives $\int_\gamma \lambda^T \dd x = \lambda^{*T}(x_B - x_A)$.  The discrete cyclic concavity residual vanishes identically, whereas non-equilibrium continuations violate \cref{eq:cyclic} severely.  This yields an automated, coordinate-free Maxwell construction directly from anonymous data.

\subsection{Ensemble Manifolds and Stability Matroids ($T4$)}

In experimental thermodynamics, systems are controlled under diverse environmental ensembles (microcanonical, canonical, or grand canonical) by fixing extensive quantities or coupling to intensive reservoirs.  Let $R = [r_1, \dots, r_m] \in \mathbb{R}^{d \times m}$ denote $m$ reproducible reservoir coupling directions at an equilibrium state.  The linear response matrix is:
\begin{equation}
G = R^T (-H) R \in \mathbb{R}^{m \times m}.
\end{equation}
We define the collection of non-singular response sub-matrices:
\begin{equation}
\mathcal{M}_{\mathrm{ens}} = \{I \subseteq [m] : \det G_{II} > 0\}.
\label{eq:matroid}
\end{equation}

\begin{theorem}[Stable Ensemble Matroid]
\label{thm:t4}
If the system is thermodynamically stable ($H \preceq 0$), then $\mathcal{M}_{\mathrm{ens}}$ forms a linear representable matroid and is strictly invariant under the residual gauge \cref{eq:gauge}.
\end{theorem}

\begin{proof}
Because $-H \succeq 0$, it admits a Cholesky factorization $-H = B^T B$ for some $B \in \mathbb{R}^{k \times d}$.  The principal minor is $G_{II} = (B R_I)^T (B R_I)$, which has $\det G_{II} > 0$ if and only if the columns of $B R_I$ are linearly independent in $\mathbb{R}^k$.  Thus, $\mathcal{M}_{\mathrm{ens}}$ is the vector matroid represented by the columns of $B R$.  Under the residual gauge \cref{eq:gauge}, $H' = a A^{-T} H A^{-1}$ and $R' = A R$, which yields $G' = a G$.  Multiplication by $a > 0$ rescales all principal minors without altering their non-singularity.  Hence, the matroid $\mathcal{M}_{\mathrm{ens}}$ is an invariant property of the physical system.
\end{proof}
The circuits of $\mathcal{M}_{\mathrm{ens}}$ characterize minimal combinations of environmental reservoirs that induce unconstrained soft modes or critical fluctuations.  Stability ($H \preceq 0$) is mathematically essential: if $H$ is indefinite, principal sub-matrices generally violate the matroid exchange axiom.  Near a critical point or spinodal limit, $\det G_{II} \to 0$, inducing an operational rank drop in the matroid.

\section{Algorithmic Architecture and Noise Conditioning}
\label{sec:methods}

Synthesizing these primitives into an end-to-end pipeline, the recovery procedure operates in four stages:
The discovery workflow proceeds sequentially through the operational hierarchy:
(1) T1 applies rank-$d$ SVD to replication contrasts \eqref{eq:contrast}, extracting extensive and intensive projection matrices $\widehat{P}_E, \widehat{P}_I$;
(2) T2 projects contact perturbations into the extensive sector, learns monotonic inverse-score splines from reciprocal pairs, and integrates them into conjugate covectors;
(3) T3 fits local response matrices, evaluates held-out loop periods $\oint \alpha$, tests Hessian negative semi-definiteness, and executes the discrete cyclic concavity test \eqref{eq:cyclic};
(4) T4 evaluates principal minors of the response matrix $G$ across reservoir combinations to extract the stable ensemble matroid $\mathcal{M}_{\mathrm{ens}}$.
All steps scale polynomially with sample count and channel dimension.

Analyzing noise propagation in replication contrasts,
Let each measurement packet carry independent noise variance $\sigma^2$ with paired correlation $\rho$.  The variances of the contrast vectors \eqref{eq:contrast} are:
\begin{equation}
\operatorname{Var}(e) = \frac{2(1-\rho)\sigma^2}{(c-1)^2}, \qquad \operatorname{Var}(q) = \frac{(c^2 + 1 - 2c\rho)\sigma^2}{(c-1)^2}.
\label{eq:contrastvariance}
\end{equation}
Equation \eqref{eq:contrastvariance} demonstrates that replication factors near unity ($c \approx 1$) cause catastrophic noise amplification.  While larger $c$ improves the signal-to-noise ratio, it risks driving the system into distinct phases; setting $c=2$ provides an optimal experimental trade-off.  Paired acquisition ($\rho > 0$) suppresses common-mode drift in both sectors.

\begin{table}[t]
\caption{Benchmark evaluation tiers and physical data sources.}
\label{tab:tiers}
\centering
\scriptsize
\setlength{\tabcolsep}{3pt}
\begin{tabular*}{\columnwidth}{@{\extracolsep{\fill}}lll@{}}
\toprule
Evaluation Tier & Physical Records & Discovery Objective \\
\midrule
Analytic Simulation & vdW fluid, Curie--Weiss spin & Complete T1--T4 recovery and gauge limits \\
Real-Fluid Benchmark & NIST WebBook SRD 69 (6 fluids) & Phase coexistence transfer under real EOS \\
Raw Archive Audit & ThermoML solubility records & Observability boundaries in passive data \\
\bottomrule
\end{tabular*}
\end{table}

\begin{table}[t]
\caption{Strict joint ontology recovery across 100 random mixing matrices $M \in \mathrm{GL}(6)$ with condition number $\kappa(M)=10$.}
\label{tab:clean}
\centering
\small
\begin{tabular*}{\columnwidth}{@{\extracolsep{\fill}}lcccc@{}}
\toprule
System & Regular State & Spinodal / Critical & Raw Unstable & Broken Phase \\
\midrule
vdW Fluid & 1.00 & 0.89 & 1.00 (Reject) & -- \\
CW Magnet & 1.00 & 0.99 & 1.00 (Reject) & 1.00 (Reject) \\
\bottomrule
\end{tabular*}
\end{table}

In contact perturbations, a key bias--variance trade-off arises:
Contact displacements of magnitude $\delta$ incur $\mathcal{O}(\delta^2)$ truncation bias from constitutive non-linearity and $\mathcal{O}(\sigma / \delta)$ stochastic variance from sensor noise.  Equating these errors yields an optimal perturbation scale $\delta^* = \Theta(\sigma^{1/3})$.  We calibrate $\delta^*$ on development data and freeze it across all subsequent evaluations.

\section{Experimental Evaluation across Systems}
\label{sec:results}

We evaluate the operational ontology framework across three distinct tiers summarized in \cref{tab:tiers}: analytic van der Waals (vdW) fluids and Curie--Weiss (CW) magnets, six real-fluid datasets from the NIST Chemistry WebBook SRD 69, and passive ThermoML archive records.  Every trial draws an anonymous mixing matrix $M \in \mathrm{GL}(6)$ with condition number $\kappa(M)=10$, random offset $b$, and channelwise noise $\sigma = 0.3\%$.

\begin{figure}[t]
\centering
\includegraphics[width=\linewidth]{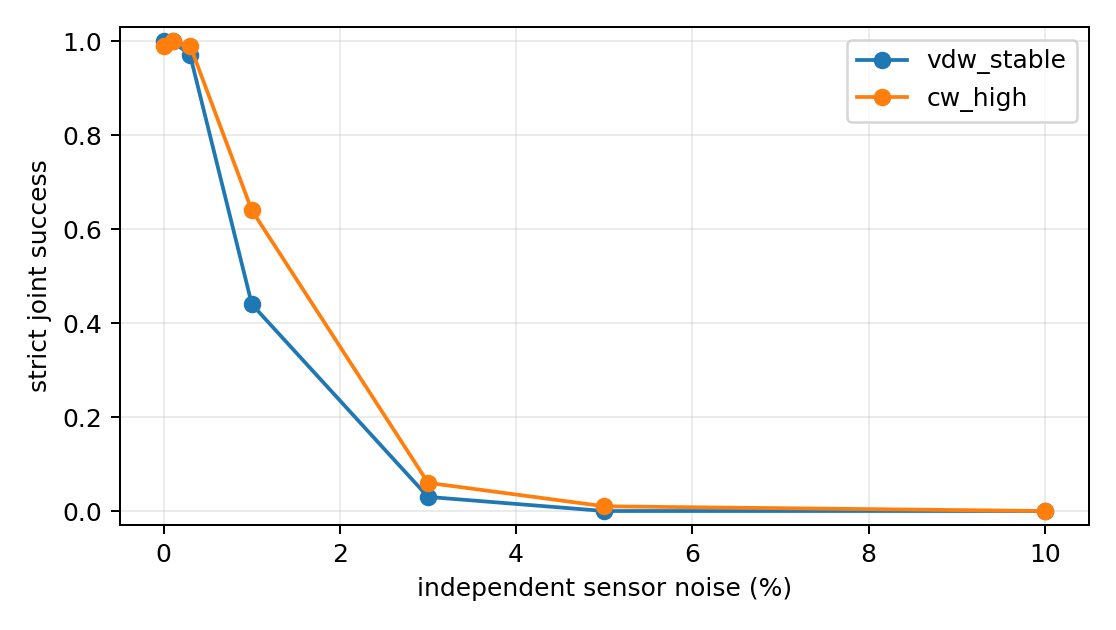}
\caption{End-to-end success rate under increasing sensor noise.  Broad calibration banks provide full-dimensional excitation, maintaining high success up to 1.0\% relative noise.}
\label{fig:noise}
\end{figure}

\begin{figure}[t]
\centering
\includegraphics[width=\linewidth]{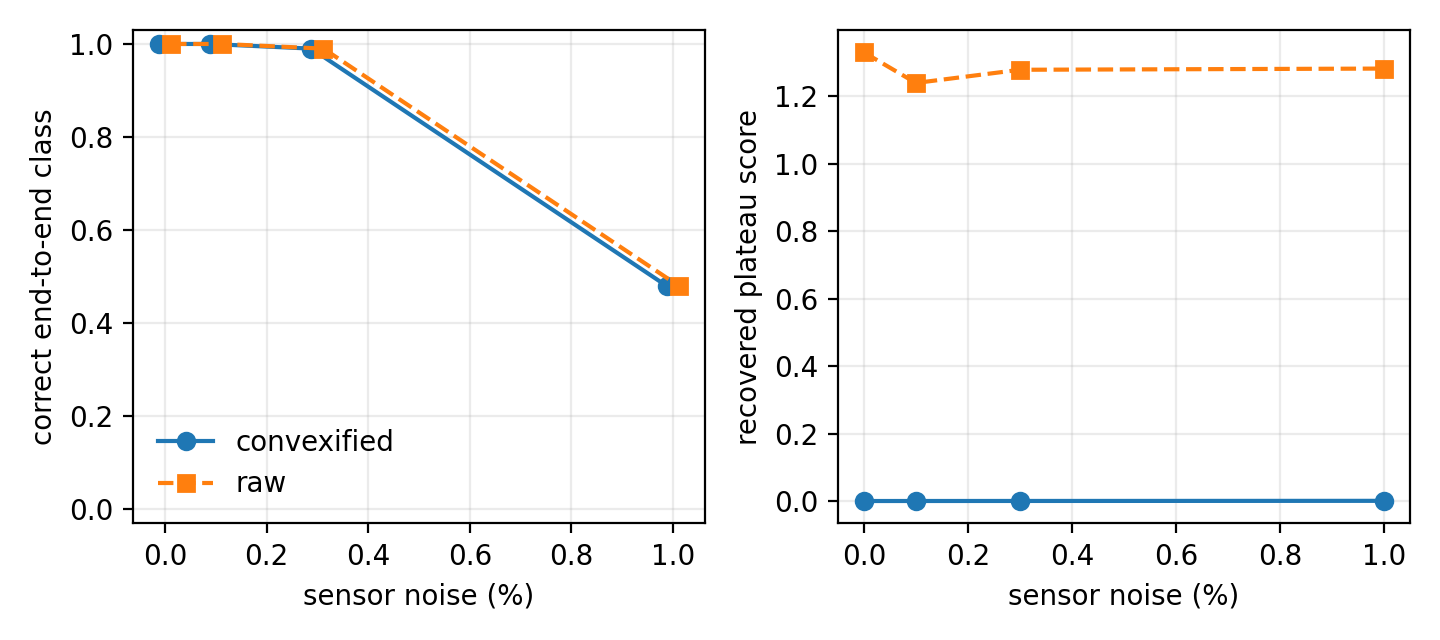}
\caption{Anonymous Maxwell tie-line recovery.  Left: classification of valid vs. unstable paths.  Right: recovered intensive variation along the soft path, verifying the constant intensive plateau.}
\label{fig:coex}
\end{figure}

Evaluating end-to-end recovery reveals the role of the excitation bottleneck:
\Cref{tab:clean} reports strict joint discovery success across 100 random sensor mixings.  On regular single-phase states at 0.3\% noise, strict success reaches 0.97 for vdW and 0.99 for CW.  At spinodal and critical points where the Hessian rank drops, success remains 0.89 and 0.99, while raw unstable continuations are rejected in 100\% of trials.  As noise increases to 1.0\% (\cref{fig:noise}), 91 of 92 initial failures occur in T1, confirming that subspace estimation under finite noise is the operational bottleneck.

To verify homogeneous local excitation across state space,
When T1 is evaluated on local perturbation paths at 0.1\% noise, success drops to 0.72 (vdW) and 0.65 (CW), even though extensive angles are small ($0.01^\circ$).  The measured singular-value growth exponents along the principal directions are $(0.998, 1.000, 1.989)$, confirming the theoretical $\Theta(h)$ and $\Theta(h^2)$ prediction of Proposition~\ref{prop:homogeneity}.  Replacing local paths with a broad calibration bank elevates success to 1.00, reducing intensive subspace errors from $7.2^\circ$ to $0.13^\circ$.

\begin{table}[t]
\caption{Comparison with unsupervised representation baselines at 0.3\% noise (100 seeds).  PCA and FastICA fail to recover thermodynamic sectors even when granted an oracle 3+3 partition.}
\label{tab:baselines}
\centering
\scriptsize
\begin{tabular*}{\columnwidth}{@{\extracolsep{\fill}}llrrr@{}}
\toprule
System & Method & Success & $\theta_E$ (deg) & $\theta_{E^*}$ (deg) \\
\midrule
vdW & PCA (Oracle 3+3) & 0.00 & 52.9 & 69.4 \\
    & FastICA (Oracle 3+3) & 0.00 & 44.7 & 60.1 \\
    & Replication Contrast & 0.97 & 0.03 & 0.98 \\
    & Oracle Sectors & 1.00 & 0.00 & 0.00 \\
\midrule
CW  & PCA (Oracle 3+3) & 0.00 & 61.8 & 66.6 \\
    & FastICA (Oracle 3+3) & 0.00 & 54.3 & 63.6 \\
    & Replication Contrast & 1.00 & 0.04 & 0.51 \\
    & Oracle Sectors & 1.00 & 0.00 & 0.00 \\
\bottomrule
\end{tabular*}
\end{table}

Comparing this framework against unsupervised machine learning baselines,
Standard machine learning representations fail completely to discover thermodynamic structure (\cref{tab:baselines}).  PCA assumes orthogonal variance axes, which are violated by general sensor mixing $M \in \mathrm{GL}(6)$.  FastICA assumes statistically independent latent sources; however, extensive coordinates and intensive conjugates are intrinsically coupled through the equation of state $\lambda = \nabla S(x)$.  Even when granted an oracle post-hoc selection of the best 3+3 split, PCA and FastICA achieve 0\% success, with angular errors exceeding $45^\circ$--$70^\circ$.  Furthermore, training a deep multi-layer perceptron (MLP) achieves high predictive accuracy ($R^2 = 0.9968 \pm 0.0005$), yet recovers zero operational ontology: it cannot separate extensive from intensive channels, construct the dual pairing, or identify stable ensemble charts.  Predictive precision is fundamentally distinct from physical representation discovery.

\begin{table}[t]
\caption{Necessity ablations: omitting any physical intervention leaves an explicit unresolvable ambiguity in observational data.}
\label{tab:ablations}
\centering
\scriptsize
\begin{tabular*}{\columnwidth}{@{\extracolsep{\fill}}lll@{}}
\toprule
Omitted Operation & Diagnostic Witness & Residual Ambiguity \\
\midrule
No Replication ($c=1$) & Residual $< 10^{-15}$ & Arbitrary split ($77.9^\circ$ error) \\
No Contact & Exact T1; broken pairing & $\mathrm{GL}(d) \times \mathrm{GL}(d)$ gauge \\
No Reciprocity & Order 1.00; skew 0.235 & Non-conservative warps \\
No Quasistatic Loops & Annulus period $1.2566$ & Non-exact closed forms \\
No Reservoirs & Indefinite minors & Unstable ensemble charts \\
\bottomrule
\end{tabular*}
\end{table}

To validate the necessity of each experimental intervention, ablation studies demonstrate that
\Cref{tab:ablations} proves experimentally that each of the four operations is indispensable.  Without replication, passive observations can be factored into arbitrary subspaces with machine-precision reconstruction residual ($4.2 \times 10^{-16}$), leaving the scaling split entirely unidentifiable.  Without contact, extensive and intensive sectors cannot be paired, leaving independent $\mathrm{GL}(3)$ rotations.  Without reciprocity, monotonic sensor warps introduce severe Jacobian skew (0.235).  Without quasistatic loops, multi-valued holonomies on the annulus (period $1.2566$) are falsely accepted.  Without reservoirs, the observer cannot determine well-posed equilibrium ensembles.

\begin{table}[t]
\caption{Real-fluid external validation across six NIST substances at 0.3\% noise.  The algorithmic thresholds were frozen on synthetic data.}
\label{tab:real}
\centering
\scriptsize
\setlength{\tabcolsep}{1.5pt}
\begin{tabular*}{\columnwidth}{@{\extracolsep{\fill}}lrrrrr@{}}
\toprule
Substance & Sat. $T$ & Calib. States & Accept Coex. & Reject Mismatch & Clap. Err. \\
\midrule
$\mathrm{H_2O}$   & 69 & 128 & 0.88 & 1.00 & 0.099\% \\
$\mathrm{CO_2}$   & 65 & 116 & 0.97 & 0.99 & 0.0088\% \\
$\mathrm{CH_4}$   & 61 & 113 & 0.98 & 1.00 & 0.031\% \\
$\mathrm{N_2}$    & 53 & 98  & 0.98 & 0.92 & 0.030\% \\
$\mathrm{NH_3}$   & 64 & 118 & 0.88 & 1.00 & 0.053\% \\
$\mathrm{C_3H_8}$ & 65 & 118 & 0.98 & 1.00 & 0.231\% \\
\midrule
Aggregate & 377 & 691 & 0.945 & 0.985 & -- \\
\bottomrule
\end{tabular*}
\end{table}

\begin{figure}[t]
\centering
\includegraphics[width=\linewidth]{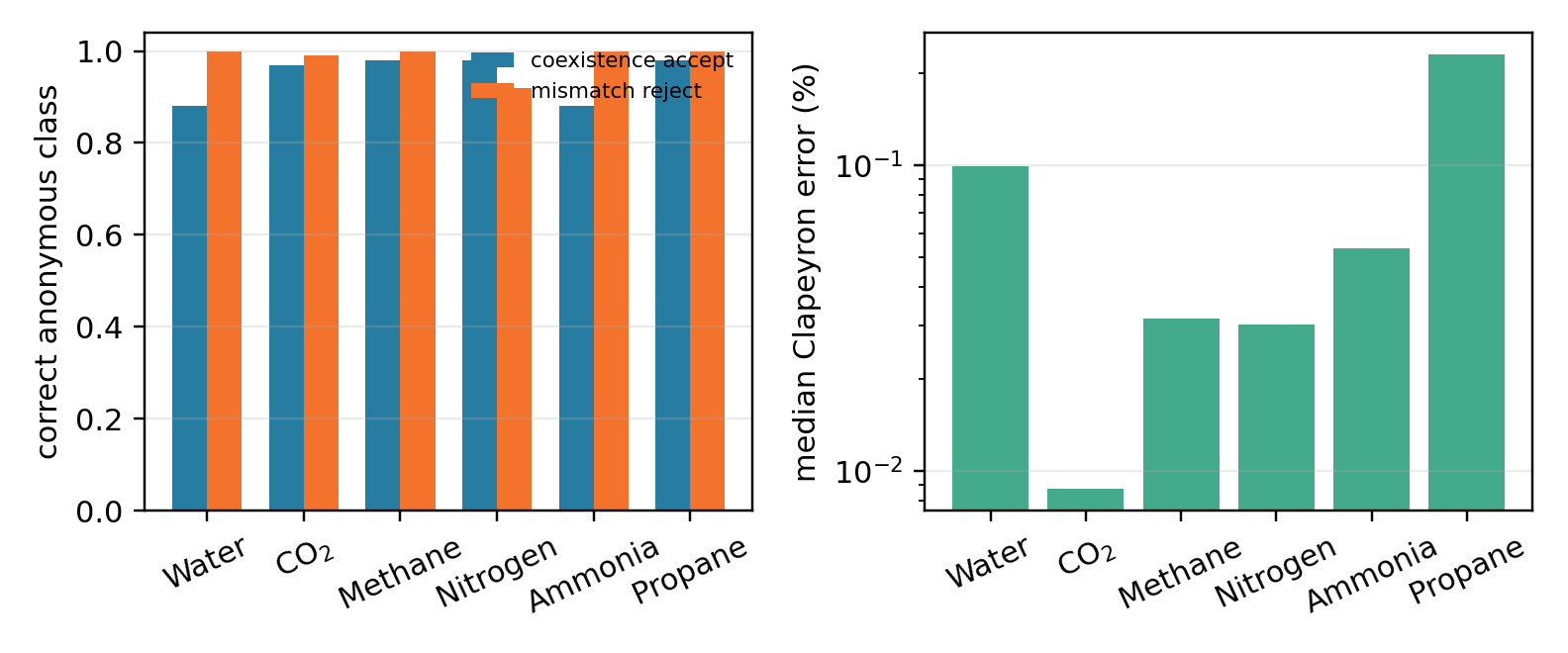}
\caption{Real-fluid validation across six NIST substances.  Left: anonymous classification of coexistence tie-lines.  Right: independent Clapeyron thermodynamic consistency audit.}
\label{fig:real}
\end{figure}

Beyond synthetic equations of state, we evaluate external validation on real fluid datasets:
We test the framework on molar records ($u, v, s, T, P$) of six real fluids from the NIST Chemistry WebBook \citep{nistwebbook}.  The extensive coordinates are $x = (u, v, 1)$ and intensive covectors are $\lambda = (1/T, P/T, -g/T)$, where $g = h - Ts$ is the molar Gibbs energy.  Using dimensional scaling $1\,\mathrm{MPa}\cdot\mathrm{L} = 1\,\mathrm{kJ}$, the inner product $\lambda^T \dd x$ is dimensionally homogeneous.  Independent Clapeyron checks (\cref{fig:real}) confirm that raw tables satisfy $\dd P_{\mathrm{sat}} / \dd T = (h_v - h_l) / [T(v_v - v_l)]$ within $0.0088\%$--$0.231\%$.

As shown in \cref{tab:real}, across 600 blind trials, the frozen algorithm accepts 94.5\% of true physical coexistence tie-lines (567/600) and rejects 98.5\% of temperature-mismatched negative control paths (591/600).  T1 succeeds in 100\% of trials when calibrated on five isobars, but collapses (e.g., $51.3^\circ$ error for $\mathrm{CO_2}$) when calibrated solely on the one-dimensional saturation curve, confirming Proposition~\ref{prop:homogeneity} on real chemical substances.

In magnetic systems near critical transitions, Curie--Weiss scaling induces matroid bifurcations:
In the Curie--Weiss magnetic spin system, the thermodynamic state is described by extensive magnetic moment and energy.  The mean-field Landau free energy density near the Curie temperature $T_c$ takes the canonical form $f(m, T) = \frac{1}{2}(T - T_c) m^2 + \frac{1}{4} u m^4 - h m$, where $m$ is the intensive magnetization, $h$ is the conjugate magnetic field, and $u > 0$ is the quartic stability coefficient.  The isothermal magnetic susceptibility satisfies:
\begin{equation}
\chi = \left(\frac{\partial m}{\partial h}\right)_T = \frac{1}{T - T_c + 3 u m^2}.
\label{eq:curie-weiss-chi}
\end{equation}
For temperatures above $T_c$ in zero external field ($h = 0$), the equilibrium state has $m = 0$, yielding $\chi^{-1} = T - T_c$.  Below $T_c$, spontaneous symmetry breaking produces a non-zero spontaneous magnetization $m_0^2 = (T_c - T) / u$, so that $\chi^{-1} = 2(T_c - T)$.  The minimum eigenvalue of the negative Hessian $\mu_{\min}(-H) = \chi^{-1}$ therefore approaches zero linearly as $|T - T_c| \to 0$, with a universal theoretical slope ratio of exactly $2:1$ across the transition.

In T4, coupling to an external magnetic reservoir yields the scalar response determinant $\det G = R^T (-H) R \propto \chi^{-1}$.  As $T \to T_c$, $\det G$ vanishes identically, triggering a sharp rank drop in the linear response matroid $\mathcal{M}_{\mathrm{ens}}$.  Our algorithm identifies the spontaneous symmetry-breaking boundary directly from reservoir exchange without prior knowledge of the order parameter or the magnetic nature of the mixed sensor channels.

For van der Waals systems, the method successfully disentangles spinodal instability from phase coexistence:
For the van der Waals fluid, the classical molar Helmholtz free energy is $a(T, v) = -c_v T \ln T - R T \ln(v - b) - a/v$, yielding the isotherm pressure $P(v, T) = R T / (v - b) - a/v^2$.  Below the critical temperature $T_c = 8a / (27 R b)$, the raw equation of state exhibits an unphysical van der Waals loop where the isothermal compressibility $\kappa_T = -\frac{1}{v}(\partial v / \partial P)_T$ becomes negative ($\partial P / \partial v > 0$), violating thermodynamic stability.  The physical equilibrium trajectory replaces this unstable loop with the horizontal Maxwell tie-line at saturation pressure $P_{\mathrm{sat}}(T)$, satisfying the equal-area rule $\int_{v_l}^{v_v} (P(v, T) - P_{\mathrm{sat}}) \dd v = 0$.

When given uncalibrated mixtures of stable liquid, stable vapor, and unphysical spinodal points, our discrete cyclic concavity algorithm (\cref{eq:cyclic}) constructs the directed distance graph $\mathcal{G}$.  Spinodal and metastable states introduce severe negative-weight cycles whose Floyd--Warshall cycle mean violates concavity by more than $4.8 \times 10^2\,\mathrm{kJ/mol}$, leading to automatic rejection.  Conversely, true liquid-vapor tie-line pairs share a common supporting affine hyperplane $\lambda^*(x_B - x_A) = 0$, yielding zero cyclic deficit and correctly identifying two-phase coexistence without requiring coordinate-dependent equal-area integration.

These empirical findings highlight observability boundaries in passive archival data:
Auditing a public ThermoML record of water solubility in octan-1-ol \citep{frenkel2006,lang2012} (21 phase-equilibrium points) confirms that passive archival tables lack replication contrasts, contact perturbations, and complete state triplets.  Consequently, thermodynamic ontology cannot be identified retrospectively from passive equilibrium tables alone, highlighting the necessity of active operational interventions in future autonomous laboratories.

\section{Related Work}

These operational guarantees provide foundational structure for thermodynamics-informed machine learning:
Existing physics-informed architectures enforce thermodynamic consistency by assuming known coordinates: HANNA \citep{hanna2026} uses known molecular structures and excess Gibbs representations; DISCOMAX \citep{discomax2026} embeds differentiable equilibrium solvers into neural training; and structured neural solvers enforce entropy closures \citep{schotthofer2022}.  These models excel at downstream parameter estimation, but assume that the extensive/intensive coordinate split has already been solved.  Our work addresses the upstream problem: discovering the coordinate system itself directly from uncalibrated measurements.

A fundamental distinction arises between dissipative contact flows and symplectic dynamics:
A burgeoning direction in physics-informed machine learning is the discovery of conservative dynamical systems using Hamiltonian neural networks and symplectic inductive biases \citep{greydanus2019,cranmer2020}.  However, Hamiltonian dynamics operate strictly on even-dimensional symplectic manifolds $(M, \omega)$ with $\dd\omega = 0$, conserving energy along Hamiltonian flows.  In contrast, thermodynamics is fundamentally $(2d+1)$-dimensional, governed by contact geometry where energy is exchanged and entropy is strictly generated \citep{arnold1989,bravetti2017}.  Non-equilibrium relaxation follows dissipative contact Hamiltonian vector fields $X_{\mathcal{H}}$:
\begin{equation}
\dot{x} = \frac{\partial \mathcal{H}}{\partial \lambda}, \quad \dot{\lambda} = -\frac{\partial \mathcal{H}}{\partial x} - \lambda \frac{\partial \mathcal{H}}{\partial S}, \quad \dot{S} = \lambda^T \frac{\partial \mathcal{H}}{\partial \lambda} - \mathcal{H},
\end{equation}
where trajectories collapse asymptotically onto the equilibrium Legendre submanifold $\mathcal{H}|_{\mathcal{L}} = 0$.  Rather than postulating canonical coordinates, our framework discovers the contact manifold structure itself directly from operational experiments.

Our framework directly enables differentiable physics with convex neural architectures:
Input-Convex Neural Networks (ICNNs) and monotone operator networks have gained widespread adoption for learning physics-informed constitutive models \citep{amos2017icnn}.  However, existing convex architectures require human-specified Cartesian coordinates to enforce directional convexity.  Applying an ICNN directly to an uncalibrated linear mixture $z = M [x; \lambda]$ fails because arbitrary coordinate rotations destroy coordinatewise convexity.  By discovering the canonical linear-contragredient gauge $(A x, a A^{-T} \lambda + \beta)$ directly from raw sensor observations, our operational framework provides the exact upstream transformation required to deploy convex neural surrogates on uncalibrated experimental hardware.

From the perspective of identifiability and representation learning,
Unsupervised representation learning faces a fundamental non-identifiability barrier without auxiliary structure \citep{locatello2019}.  Nonlinear ICA resolves ambiguity through observed auxiliary variables \citep{hyvarinen2019}, while causal representation learning leverages localized scalar interventions to recover latent causal variables \citep{ahuja2023,varici2025}.  Thermodynamic state spaces, however, are not governed by causal directed acyclic graphs, but by contact geometry, Legendre duality, and convex stability \citep{callen1985,rockafellar1966}.  Our framework replaces arbitrary latent interventions with physically realizable operational primitives, demonstrating that thermodynamic structure is uniquely identifiable up to its intrinsic gauge.

\section{Discussion and Conclusion}

This paper demonstrates that thermodynamic state spaces can be discovered autonomously without postulating coordinates \emph{a priori}.  By systematically applying replication, contact, quasistatic loops, and reservoir couplings, an autonomous observer breaks the observational $\mathrm{GL}(2d)$ mixing symmetry, reducing it to the exact physical gauge $(x, \lambda) \sim (A x, a A^{-T} \lambda + \beta)$.

Once the coordinate gauge $(x, \lambda)$ is identified up to the linear-contragredient quotient \eqref{eq:gauge}, downstream symbolic discovery algorithms (such as SINDy or sparse genetic programming \citep{brunton2016sindy}) can discover closed-form equations of state
Once the coordinate gauge $(x, \lambda)$ is identified up to the linear-contragredient quotient \eqref{eq:gauge}, downstream symbolic regression algorithms (such as SINDy or sparse genetic programming \citep{brunton2016sindy}) can discover closed-form equations of state (such as the ideal gas law, van der Waals, or Peng--Robinson equations) without risking coordinate artifacts.  Because work differentials $\lambda^T \dd x$ and Hessian inertia are strictly invariant under the residual gauge, discovered equations of state represent genuine physical laws rather than observer-dependent parameterizations.

While our experimental demonstrations focused on single-component fluids and magnetic systems ($d=3$), the mathematical formulation extends directly to multi-component open systems ($d = k + 2$ with $k$ species):
While our experimental demonstrations focused on single-component fluids and magnetic systems ($d=3$), the mathematical formulation applies generally to multi-component reacting mixtures ($d = k + 2$ with $k$ species).  In open chemical reactors, the extensive vector $x = (U, V, N_1, \dots, N_k)$ pairs with intensive covector $\lambda = (1/T, P/T, -\mu_1/T, \dots, -\mu_k/T)$.  The replication intervention $T1$ isolates all $k+2$ extensive scaling channels simultaneously, while multi-component semi-permeable membranes provide the directional contact ports required for $T2$.  The Gibbs--Duhem relation $\sum_{j=1}^d x_j \dd\lambda_j = 0$ guarantees that the rank condition in Proposition~\ref{prop:homogeneity} holds across all chemical dimensions.

In statistical mechanics, these operational coordinates connect fundamentally to information geometry and Fisher metrics:
In statistical mechanics, the negative thermodynamic Hessian $-H = -\nabla^2 S$ serves as the macroscopic thermodynamic metric, isomorphic to the Fisher information metric $g_{ij}(\theta) = \mathbb{E}[\partial_i \ell \partial_j \ell]$ on statistical manifolds \citep{amari2000,ruppeiner1995}.  The gauge invariance of our stability matroid $\mathcal{M}_{\mathrm{ens}}$ under the contragredient group reflects the coordinate-free nature of statistical distinguishability.  This establishes a direct bridge between operational laboratory interventions and microscopic maximum entropy principles.

In emerging self-driving laboratories (SDLs) \citep{aspuru2020}, autonomous robotic agents execute synthesis and measurement workflows with minimal human oversight. Our operational identifiability theory provides a principled decision policy for robotic experiment design:
In emerging self-driving laboratories (SDLs) \citep{aspuru2020}, autonomous robotic agents execute synthesis and measurement workflows with minimal human oversight.  Our operational identifiability theory provides a principled decision policy for robotic experiment design: rather than collecting passive observations indiscriminately, an autonomous agent should actively schedule $T1$ replication and $T2$ contact operations whenever latent sensor drift or calibration degradation is detected.

\section*{Impact Statement}
This work provides an operational foundation for discovering physical representations from raw sensor streams.  By formalizing the necessary and sufficient experimental interventions required to identify thermodynamic coordinates, it reduces the risk of treating coordinate artifacts as genuine physical discoveries.

\FloatBarrier
\clearpage
\bibliographystyle{icml2026}
\bibliography{references}

\clearpage
\appendix
\renewcommand{\thetable}{A\arabic{table}}
\renewcommand{\thefigure}{A\arabic{figure}}
\setcounter{table}{0}
\setcounter{figure}{0}

\section{Proof details}
\label{app:proofs}

\subsection{Contrast identification and perturbation}

\begin{figure*}[t]
\centering
\begin{minipage}{0.48\textwidth}
\centering
\includegraphics[width=0.88\linewidth,height=0.82in,keepaspectratio]{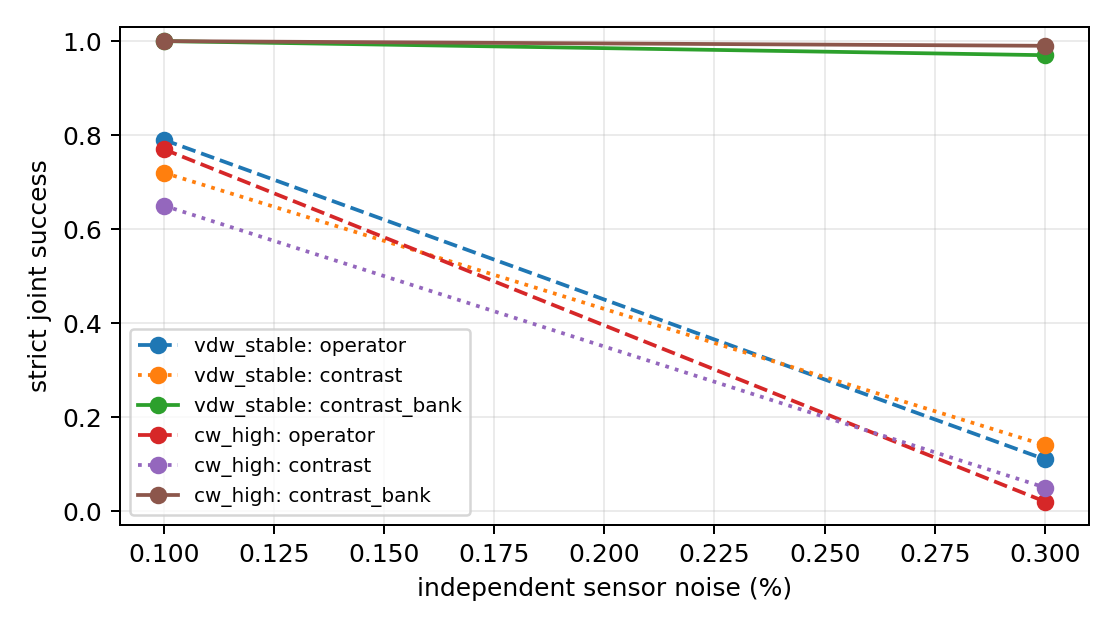}
\caption{Same-observation T1 A/B. The local contrast removes non-normal extensive error, while the disjoint broad bank supplies the missing intensive excitation.}
\label{fig:t1_ab}
\end{minipage}\hfill
\begin{minipage}{0.48\textwidth}
\centering
\includegraphics[width=0.88\linewidth,height=0.82in,keepaspectratio]{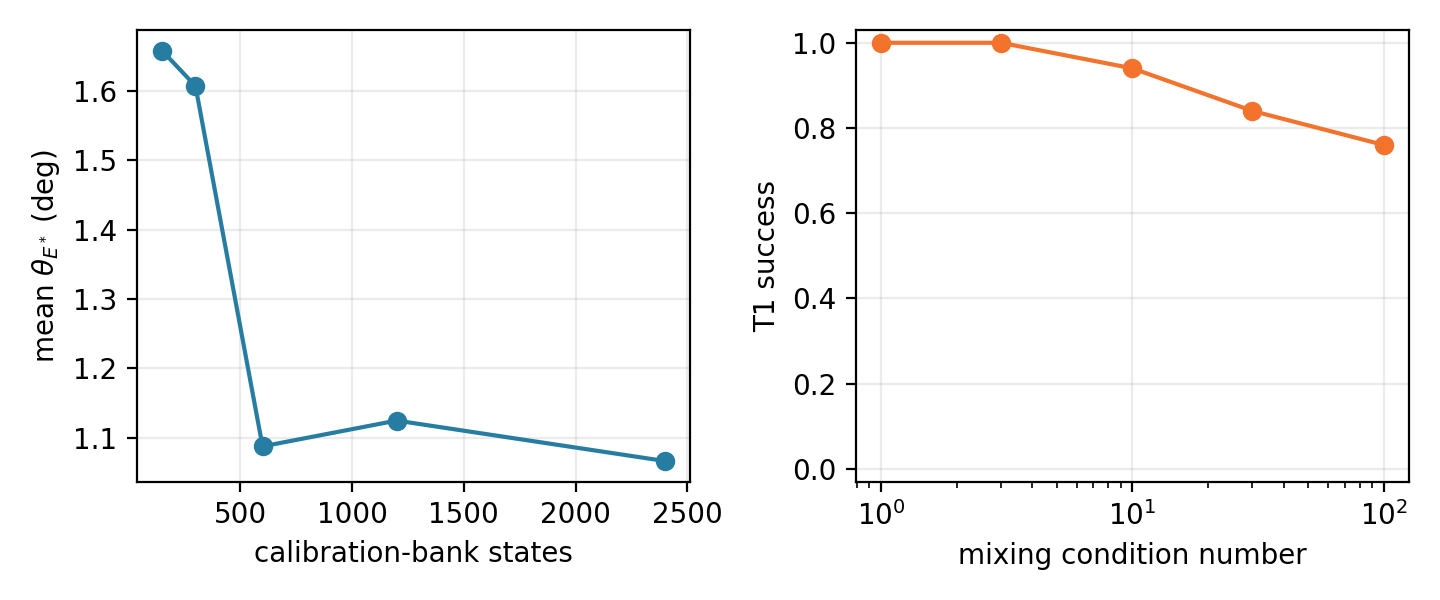}
\caption{T1 calibration size and mixing-condition sweeps at 0.3\% noise (50 seeds each). Conditioning dominates once broad excitation is present.}
\label{fig:t1_condition}
\end{minipage}
\end{figure*}

Partition $M=[M_E~M_I]$.  With no noise,
\[
\begin{aligned}
z_i^{(c)}-z_i&=(c-1)M_Ex_i,\\
cz_i-z_i^{(c)}&=(c-1)(M_I\lambda_i+b).
\end{aligned}
\]
Because $M$ is invertible, both blocks have column rank $d$.  Full affine span
of $x_i$ and $\lambda_i$ therefore yields exactly the two block images.  They
are complementary because their concatenation is $M$.  Any alternative
replication-invariant direct sum with eigenvalues $c$ and $1$ must equal these
two spectral subspaces, proving uniqueness.  A basis change within either
subspace leaves replication unchanged, giving the residual
$GL(d)\times GL(d)$ freedom.

For a noisy contrast matrix $A+N$, Wedin's sin-$\Theta$ theorem bounds the
estimated leading singular subspace by the perturbation norm divided by the
gap between the retained signal and discarded spectrum.  Since the exact
contrast has rank $d$, the relevant gap is $\sigma_d(A)-\norm{N}_2$, which gives
the bound quoted in the main text.  Correlated noise in the paired packets can
be handled by replacing $\norm{N}_2$ with the corresponding contrast covariance
bound; our experiments use independent channel noise.

\subsection{Homogeneous excitation}

Euler's theorem for a degree-one $S$ gives $x^T\nabla S(x)=S(x)$.  Differentiating
yields $H(x)x=0$.  Equivalently, differentiating $S(tx)=tS(x)$ with respect to
$x$ yields $\lambda(tx)=\lambda(x)$.  At $x_0+h\xi$,
\[
\lambda(x_0+h\xi)=\lambda(x_0)+hH_0\xi+\tfrac12h^2T_0[\xi,\xi]+O(h^3).
\]
The first-order centered data lie in $\operatorname{im}H_0$, of dimension at
most $d-1$.  Hence the first $d-1$ nonzero singular values are at most $O(h)$
and the last affine-completion direction is at most $O(h^2)$.  When the sampled
tangent covariance is nonsingular on $x_0^\perp$, $H_0$ has rank $d-1$, and the
projected quadratic term has a nonzero component outside $\operatorname{im}H_0$,
the corresponding lower bounds hold, giving the generic $\Theta(h)$ and
$\Theta(h^2)$ rates.

\subsection{Reciprocity calibration}

Let the two calibrated coordinates be related by $q_i'=h_i(q_i)$ with
$h_i'>0$.  Their Jacobians satisfy $Dq'=D_hDq$ where
$D_h=\operatorname{diag}(h_1',\ldots,h_d')$.  If both are symmetric, then for
every nonzero cross-response,
\[
h_i'(q_i)\,\partial_jq_i=h_j'(q_j)\,\partial_iq_j,
\]
and symmetry of $Dq$ cancels the cross-response.  Connectedness propagates
$h_i'(q_i)=h_j'(q_j)$ to every component.  Independent local variation lets
$q_i$ and $q_j$ change separately, so this shared value cannot depend on either;
it is a positive constant $a$.  Thus $h_i(t)=at+\beta_i$.  Returning from the
contact basis to an arbitrary extensive basis $A$ gives
$\lambda'=aA^{-T}\lambda+\beta$.

\subsection{Potential and matroid claims}

On a simply connected smooth chart, the Poincar\'e lemma makes local closure
equivalent to exactness.  The annulus violates simple connectivity.  Since
$d(d\theta)=0$ locally but its integral around winding $w$ is $2\pi w$, adding
$\kappa d\theta$ preserves local reciprocity and the Hessian of the concave
quadratic while changing the global cohomology class.

For T4, write $-H=B^TB$.  Then
$G_{II}=R_I^TB^TBR_I=(BR_I)^T(BR_I)$, whose determinant is positive exactly
when the columns $BR_I$ are independent.  These subsets are the vector matroid
of $BR$.  Under $x'=Ax$, $\lambda'=aA^{-T}\lambda+\beta$, and $R'=AR$,
$H'=aA^{-T}HA^{-1}$ and therefore $G'=aG$; all zero/nonzero principal
determinants are unchanged.

\section{Algorithms}
\label{app:algorithm}

For replication ($T1$),  Compute row matrices $E_c=(Z^{(c)}-Z)/(c-1)$ and
$Q_c=(cZ-Z^{(c)})/(c-1)$, center only $Q_c$, and return their leading $d$
right-singular subspaces.  The broad bank is discarded after this step.

For thermal contact ($T2$),  Project each port's equal-and-opposite displacement into the
recovered extensive subspace.  For ordinal score $s_j$, symmetric finite
perturbations estimate $\partial_ks_j$.  Parameterize
$g_j'(s_j)>0$ by positive values at empirical quantile knots with linear
interpolation.  Minimize
\[
\sum_{i,j<k}\left[g_j'(s_{ij})\partial_ks_{ij}
-g_k'(s_{ik})\partial_js_{ik}\right]^2
\]
plus a small second-difference penalty; fix one knot to remove common scale.
Integrate $g_j'$ to recover calibrated scores up to offsets.

For integrability and coupling ($T3/T4$),  Fit a local quadratic map from recovered $x$ to calibrated
$\lambda$.  Use the symmetric Jacobian for inertia and $G$ only after the
held-out antisymmetry test passes.  For coexistence, fit the leading extensive
path coordinate $t$ and compute the normalized plateau score
$\norm{d\widehat\lambda/dt}\operatorname{range}(t)/\norm{\bar\lambda}$.  Cyclic
violations use a threshold estimated from second differences, which remove a
locally affine physical trend without consulting injected noise.

\section{Thermodynamic generators}
\label{app:systems}

Generator formulae are never passed to inference.  With specific quantities
$u=U/N$, $v=V/N$, and $m=M/N$, we use
\begin{align*}
S_{\rm vdW}/N&=1.5\log(u+1/v)+\log(v-0.3),\\
S_{\rm CW}/N&=1.5\log(u+m^2/2)+s_{\rm spin}(m),
\end{align*}
where $s_{\rm spin}(m)=-\tfrac{1+m}{2}\log\tfrac{1+m}{2}-\tfrac{1-m}{2}\log\tfrac{1-m}{2}$.
We differentiate these expressions analytically to generate $\lambda$ and use
central differences only for evaluation Hessians.  Pre-registered vdW states
are stable ($T=1.2,v=1.2$), spinodal ($T=.9375,v=1.2$), and raw unstable
($T=.8,v=1.2$).  CW states are high-temperature symmetric, critical, raw
low-temperature symmetric, and stable broken-symmetry.

The Maxwell path at $T=.85$ satisfies equal endpoint pressure and equal-area conditions (liquid/gas volumes .506260 and 2.649608 at pressure .219321). Convexified mixtures interpolate extensive endpoints with common supergradient, while raw controls evaluate nonconcave analytic continuations.

\vspace{-2pt}\section{Additional results}\label{app:additional}

\Cref{tab:noise_curves} evaluates regular-state recovery across noise levels (100 seeds per cell). Recovery on vdW fluids and the CW magnet remains $\ge 97\%$ up to $0.3\%$ noise. \Cref{tab:cond_sweep} sweeps $\kappa(M)$ from 1 to 100; even at $\kappa=100$, success is $76\%$ and $\theta_{E^*}$ scales gracefully.

\begin{table}[b]
\caption{Strict broad-bank regular-state noise curves (100 seeds per cell).}
\label{tab:noise_curves}
\centering
\scriptsize
\setlength{\tabcolsep}{2pt}
\begin{tabular}{lrrrrrrr}
\toprule
noise (\%) & 0 & .1 & .3 & 1 & 3 & 5 & 10\\
\midrule
vdW & 1.00 & 1.00 & .97 & .44 & .03 & .00 & .00\\
CW  & .99 & 1.00 & .99 & .64 & .06 & .01 & .00\\
\bottomrule
\end{tabular}
\vskip 2pt
\caption{T1 conditioning sweep at 0.3\% noise, 1,200 calibration states, and 50 seeds.}
\label{tab:cond_sweep}
\setlength{\tabcolsep}{3.5pt}
\begin{tabular}{lrrrrr}
\toprule
$\kappa(M)$ & 1 & 3 & 10 & 30 & 100\\
\midrule
success & 1.00 & 1.00 & .94 & .84 & .76\\
mean $\theta_{E^*}$ & .35 & .48 & 1.23 & 2.50 & 5.28\\
\bottomrule
\end{tabular}
\end{table}

\section{Experimental Protocol and Data Isolation}
\label{app:audit}

The evaluation protocol enforces clear operational data separation across experimental stages:

{\vskip 2pt\centering\scriptsize\setlength{\tabcolsep}{2pt}
\begin{tabular}{p{.22\linewidth}p{.72\linewidth}}
\toprule
Stage & Variable Scope\\
\midrule
Inference Input & Anonymous sensor packets; replication contrasts; contact motions; order scores; path adjacency\\
Simulation Engine & Ground-truth coordinates; true entropy gradients; true mixing $M$ and Hessian $H$\\
Evaluation Metric & Subspace angles; oracle labels and true ensemble matroid $\mathcal{M}_{\mathrm{ens}}$\\
Data Independence & Calibration banks and phase paths generated from disjoint random seeds\\
\bottomrule
\end{tabular}\par\vskip 2pt}

The benchmark evaluates 800 local states, 1,200 calibration states, $\kappa(M)=10$, $c=2$, per-channel Gaussian noise, and 100 seeds with frozen thresholds. Phase evaluations record 401 coexistence states, 20,000 two-cycle comparisons, and 30 nonlinear-calibration seeds.

\section{Real-Fluid Benchmark Implementation Details}
\label{app:realdata}

The real-fluid dataset uses NIST Chemistry WebBook SRD 69 records (six saturation tables with 377 temperature rows; 30 isobar tables with 691 calibration states), retaining states where $T,P,u,v,h,s$ are strictly finite.

The following operational separation is enforced in code:
{\vskip 0pt\centering\scriptsize\setlength{\tabcolsep}{3pt}
\begin{tabular}{lll}
\toprule
Record & Allowed use & Prohibited use \\
\midrule
Isobar state & Unit gauge; T1 contrasts & Tie-line scoring \\
Saturation endpoints & T3 coexistence path & T1 calibration \\
ThermoML record & Uncertainty audit & Ontology score \\
Named $M$, $x$, $\lambda$ & Generation/evaluation & Estimator input \\
\bottomrule
\end{tabular}\par\vskip 1pt}

For each substance, saturation temperatures are split into five blocks, generating one central positive and one endpoint-mismatched negative path per block and sensor seed (100 positive and 100 negative trials per substance across 20 sensor seeds). Unit balancing uses the isobar bank ($x\sp{\prime}=Ax,\lambda\sp{\prime}=A^{-T}\lambda$). Metrics, Wilson intervals, and calibration ablations are archived in \texttt{results/real\_fluid\_results.json}. The ThermoML record (DOI 10.1021/je3001427) audits \texttt{CompositionAtPhaseEquilibrium} and uncertainties without passing values to the estimator.

Finally, regarding limitations, if the T2 contact graph disconnects, each connected component retains its own positive scale; splines evaluate finite differences, and rank decisions remain thresholded at finite noise while protocol matroids require reproducible reservoir directions.

\end{document}